\documentclass{article}

\usepackage[preprint]{neurips_2026}

\usepackage[utf8]{inputenc} 
\usepackage[T1]{fontenc}    
\usepackage{hyperref}       
\hypersetup{hidelinks}      
\usepackage{bookmark}       
\usepackage{url}            
\usepackage{booktabs}       
\usepackage{amsfonts}       
\usepackage{nicefrac}       
\usepackage{microtype}      
\usepackage{xcolor}         
\usepackage{graphicx}
\usepackage{amsmath}
\usepackage{mathtools}      
\usepackage{amsthm}

\newtheorem{proposition}{Proposition}
\usepackage{algorithm}
\usepackage{algorithmic}
\usepackage{multirow}
\usepackage{subfigure}
\usepackage{wrapfig}
\usepackage{pifont}

\graphicspath{{media/}}

\title{PixelPrune: Pixel-Level Adaptive Visual Token Reduction via Predictive Coding}

\author{Nan Wang, Zhiwei Jin, Chen Chen, Haonan Lu \\
OPPO AI Center\\
\texttt{\{wangnan,jinzhiwei,chenchen4,luhaonan\}@oppo.com} 
}

\begin{document}

\maketitle

\begin{abstract}
Document understanding and GUI interaction are among the highest-value applications of Vision-Language Models (VLMs), yet they impose exceptionally heavy computational burden: fine-grained text and small UI elements demand high-resolution inputs that produce tens of thousands of visual tokens. We observe that this cost is largely wasteful---across document and GUI benchmarks, only 22--71\% of image patches are pixel-unique, the rest being exact duplicates of another patch in the same image. We propose \textbf{PixelPrune}, which exploits this pixel-level redundancy through predictive-coding-based compression, pruning redundant patches \emph{before} the Vision Transformer (ViT) encoder. Because it operates in pixel space prior to any neural computation, PixelPrune accelerates both the ViT encoder and the downstream LLM, covering the full inference pipeline. The method is training-free, requires no learnable parameters, and supports pixel-lossless compression ($\tau{=}0$) as well as controlled lossy compression ($\tau{>}0$). Experiments across three model scales and document and GUI benchmarks show that PixelPrune maintains competitive task accuracy while delivering up to 4.2$\times$ inference speedup and 1.9$\times$ training acceleration. Code is available at \url{https://github.com/OPPO-Mente-Lab/PixelPrune}.
\end{abstract}

\section{Introduction}
\label{sec:intro}

Vision-Language Models (VLMs)~\citep{liu2023llava, li2025llavaonevision, chen2024internvl, chen2025internvl25, bai2025qwen3vltechnicalreport} achieve impressive performance across multimodal tasks by integrating visual encoders with large language models. However, their computational cost scales directly with the number of visual tokens processed during both training and inference. This cost is especially pronounced in document and GUI scenarios, which demand high-resolution inputs for fine-grained text rendering and small UI elements. This burden is further amplified by the growing adoption of NaViT-style vision encoders~\citep{dehghani2023navit}---exemplified by the Qwen-VL series~\citep{wang2024qwen2vl, bai2025qwen25vl, bai2025qwen3vltechnicalreport, qwen2026qwen35}---which process images at native resolution and can produce tens of thousands of patches per image. For example, a typical document page in DocVQA averages 1775$\times$2082 pixels after preprocessing (Appendix~\ref{app:redundancy}), producing over 14{,}000 ViT patches ($p{=}16$) and roughly 3{,}600 LLM tokens after $2{\times}2$ merging; multi-page documents and multi-step GUI sessions multiply this further.

Existing VLM token reduction methods primarily operate at the semantic or feature level---relying on attention-based pruning~\citep{chen2024fastv, yang2025visionzip}, adaptive pruning-and-merging~\citep{shang2024llavaprumerge}, or diversity-based selection~\citep{alvar2025divprune}. Most are designed for general-domain images~\citep{shao2025survey}. In structured document and GUI domains, however, a more fundamental source of redundancy exists: the raw pixel content itself. Documents are dominated by white margins and uniform page backgrounds; GUI screenshots contain large solid-color panels, toolbars, and status bars---all of which produce many pixel-identical patches. As shown in Figure~\ref{fig:motivation}, large portions of patches correspond to such uniform regions. Our analysis across document and GUI benchmarks reveals that only 22--71\% of patches are pixel-unique within each image, depending on the domain (Appendix~\ref{app:redundancy}). This indicates that pixel-level content alone is a powerful signal for token reduction, eliminating redundancy at its source without learned representations.

\begin{figure*}[t]
	\centering
	\includegraphics[width=0.95\textwidth]{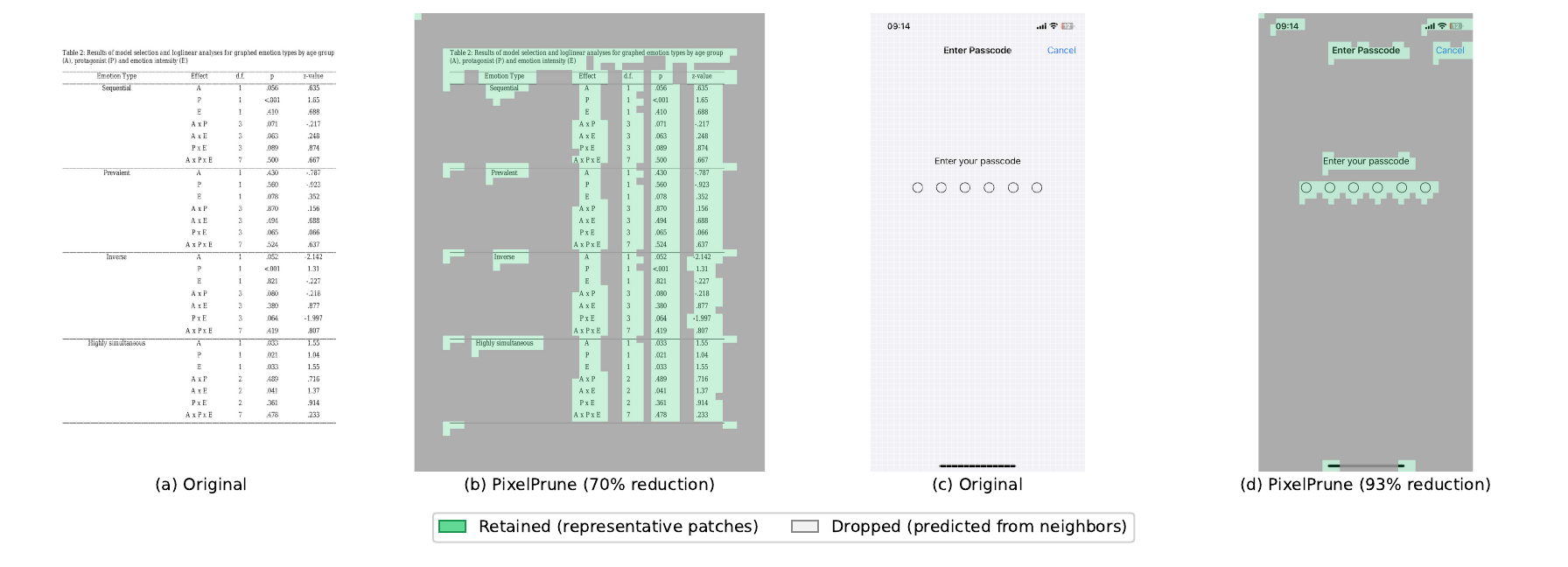}
	\caption{PixelPrune patch selection on a document image (\textit{left}) and GUI screenshot (\textit{right}). Kept patches are shown in original color; dropped patches are grayed out. Token reduction: 70\% (document) and 93\% (GUI).}
	\label{fig:motivation}
\end{figure*}

Moreover, most existing VLM acceleration methods focus exclusively on reducing the visual tokens processed by the LLM---through projector-based compression~\citep{li2023blip2, li2024tokenpacker}, attention-based pruning~\citep{chen2024fastv, he2024zipvl}, or post-ViT merging~\citep{shang2024llavaprumerge, alvar2025divprune}---while leaving the ViT encoder's cost untouched. Many are also validated primarily on fixed-resolution architectures (e.g., LLaVA~\citep{liu2023llava}), where token counts are moderate. This is a significant blind spot: in NaViT-based VLMs, the ViT input sequence scales directly with image resolution, and its self-attention cost grows quadratically with the number of patches. As Figure~\ref{fig:vision_workload} shows, the vision encoder accounts for the majority of prefill latency---up to 86\% at 4096$^2$ resolution for Qwen3-VL-2B---and this fraction grows with resolution. Few methods attempt to reduce the ViT's own input. A more direct approach is to prune visual tokens \textit{before} they enter the ViT, thereby accelerating the entire pipeline from the start.

We propose \textbf{PixelPrune}, which addresses this gap by applying predictive coding~\citep{weinberger2000loco, boutell1997png}---the same principle underlying lossless formats such as PNG and JPEG-LS---at the patch level \textit{before} vision encoding. Each patch is predicted from its causal spatial neighbors; if the prediction matches (within a threshold $\tau$), the patch is omitted as redundant. The method requires no learnable parameters and supports both pixel-lossless ($\tau{=}0$) and controlled lossy ($\tau{>}0$) compression; it can be applied training-free for document understanding, or integrated into training pipelines to simultaneously accelerate training and improve efficiency (\S\ref{sec:gui_post}--\ref{sec:jedi_scratch}). Because it operates before the ViT, PixelPrune accelerates the entire pipeline---ViT encoder, patch merger, and LLM decoder---and is fully compatible with FlashAttention~\citep{dao2022flashattention}.

Our contributions are threefold: (1)~we quantify the pervasive pixel-level patch redundancy in document and GUI images (Appendix~\ref{app:redundancy}) and propose PixelPrune, a parameter-free method that exploits this redundancy for pre-ViT token reduction with theoretical recoverability guarantees (\S\ref{sec:theory}); (2)~because PixelPrune operates before the ViT, it accelerates the \emph{entire} inference pipeline---ViT encoder, patch merger, and LLM decoder---achieving up to 4.2$\times$ end-to-end speedup; and (3)~PixelPrune integrates seamlessly into training pipelines, delivering up to 1.9$\times$ training acceleration while maintaining competitive accuracy across document and GUI benchmarks (\S\ref{sec:gui_post}--\ref{sec:jedi_scratch}).

\begin{figure*}[t]
	\centering
	\includegraphics[width=0.95\textwidth]{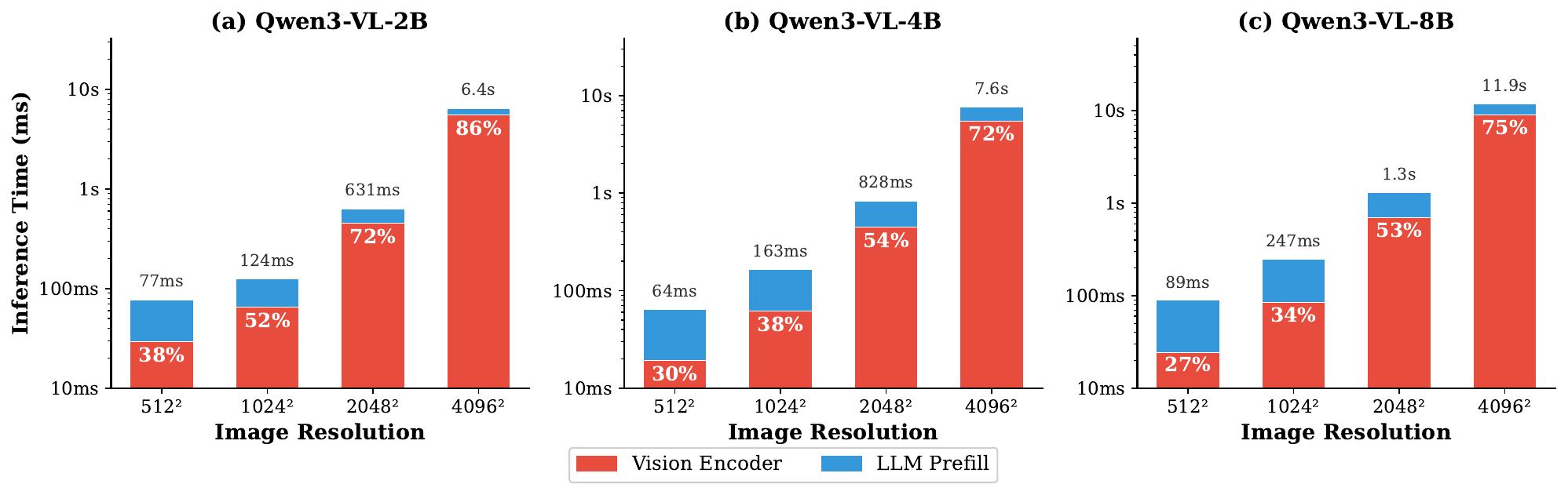}
	\caption{Prefill latency breakdown across Qwen3-VL scales (2B, 4B, 8B) at five resolutions (256$^2$--4096$^2$, log scale). Each bar splits into \textbf{Vision Encoder} (red, including ViT and Patch Merger) and \textbf{LLM Prefill} (blue). At 4096$^2$, the vision encoder accounts for 86\%, 72\%, and 75\% of total prefill time for the 2B, 4B, and 8B models respectively.}
	\label{fig:vision_workload}
	\vspace{-1em}
\end{figure*}

\section{Related Work}
\label{sec:related}
\textbf{Token Reduction in Vision Transformers.}
Vision Transformers (ViT)~\citep{dosovitskiy2020vit} treat images as patch sequences, but their quadratic attention complexity poses challenges for high-resolution inputs. MAE~\citep{he2022mae} shows that masking 75\% of patches still permits accurate reconstruction, motivating a line of token reduction methods. Token pruning approaches such as DynamicViT~\citep{rao2021dynamicvit} and EViT~\citep{liang2022evit} remove uninformative tokens via learned importance scores or attention-based reorganization, while ToMe~\citep{bolya2023tome} progressively merges similar tokens through feature averaging. These methods operate on intermediate features, requiring forward-pass computation before pruning decisions can be made. Most related to our work, RLT~\citep{choudhury2024rlt} applies run-length encoding to video transformers along the \textit{temporal} dimension. While analogous to our Raster-mode prediction, RLT is limited to 1D temporal sequences; PixelPrune operates in the 2D spatial domain and introduces a predictive coding framework that exploits richer spatial redundancy, enabling exact patch compression before ViT feature extraction.

\textbf{Efficient Vision-Language Models.}
VLMs~\citep{liu2023llava,bai2023qwenvl,wang2024qwen2vl} face efficiency challenges from high visual token counts. Existing methods reduce tokens at different pipeline stages: (1)~\textit{Projector-based}: Q-Former~\citep{li2023blip2} and TokenPacker~\citep{li2024tokenpacker} compress features before the LLM; (2)~\textit{Post-ViT}: LLaVA-PruMerge~\citep{shang2024llavaprumerge} and DivPrune~\citep{alvar2025divprune} prune or merge tokens after the vision encoder but before the LLM; (3)~\textit{Inside-LLM}: FastV~\citep{chen2024fastv}, ZipVL~\citep{he2024zipvl}, and PyramidDrop~\citep{xing2024pyramiddrop} prune tokens based on decoder attention scores, which typically require materializing attention weights---at odds with FlashAttention~\citep{dao2022flashattention}; (4)~\textit{Multi-stage}: MustDrop~\citep{liu2024mustdrop} and LUVC~\citep{zheng2025luvc} reduce tokens across multiple stages. IPCV~\citep{chen2025ipcv} takes a step further by compressing intermediate ViT features, partially reducing ViT cost but still requiring a forward pass through the early layers. All of the above operate at or after the feature level, leaving the full ViT encoder's cost---which dominates prefill latency for small-to-medium VLMs (Figure~\ref{fig:vision_workload})---largely unaddressed. PixelPrune instead exploits pixel redundancy \textit{before} vision encoding.

\textbf{Document and GUI Understanding.}
These domains feature high-resolution inputs and significant visual redundancy. While many models address the high-resolution challenge through architectural designs such as cropping, tiling, or dual encoders, few exploit the inherent pixel-level redundancy of these images for token reduction. Document models such as mPLUG-DocOwl~\citep{hu2024docowl,hu2024docowl2}, TextMonkey~\citep{liu2024textmonkey}, and GOT-OCR~\citep{wei2024got} process high-resolution images for text recognition, often relying on cropping or tiling strategies that further inflate token counts. IndexPrune~\citep{son2025indexpruning} prunes tokens before the ViT but trains a neural predictor supervised by OCR annotations; PixelPrune requires no learned parameters or external annotations. For GUI agents, CogAgent~\citep{hong2024cogagent} employs dual encoders, UI-TARS~\citep{qin2025uitars} scales up GUI interaction, Jedi~\citep{chen2025jedi} decomposes UI elements for grounding, and ShowUI~\citep{lin2024showui} reduces tokens by identifying redundant regions during self-attention. ShowUI is closest in spirit to PixelPrune: both exploit spatial redundancy in GUI images, but ShowUI applies token skipping \emph{inside} the ViT self-attention, still requiring a forward pass through the vision encoder; PixelPrune operates on raw pixels \emph{before} the ViT, yielding savings across the entire pipeline.


\section{Method}
\label{sec:method}

\subsection{Preliminaries}
\label{sec:preliminary}

We summarize the VLM visual pipeline using Qwen3-VL~\citep{bai2025qwen3vltechnicalreport} as a representative example, and highlight the architectural property that makes pre-ViT patch removal well-defined.

\paragraph{Architecture overview.}
An input image is divided into $N = (H/p) \times (W/p)$ non-overlapping patches of size $p{\times}p$ ($p{=}16$). Each patch $P_i$ is projected to a $D_v$-dimensional embedding $\mathbf{e}_i$ and paired with its 2D grid coordinate $\mathbf{p}_i = (r_i, c_i)$, which is injected via interpolated absolute positional embeddings and 2D RoPE. The sequence $\{(\mathbf{e}_i, \mathbf{p}_i)\}_{i=1}^{N}$ is processed by a $L_v$-layer ViT with $O(N^2 D_v)$ attention cost per layer, then spatially compressed by Patch Mergers ($M{\times}M$ grouping, $M{=}2$; including DeepStack mergers at selected intermediate ViT layers) to $N_m = N/M^2$ tokens, which are concatenated with $T$ text tokens to form the LLM input. In the LLM, spatial layout is encoded via MRoPE~\citep{bai2025qwen3vltechnicalreport}, which assigns each visual token a 3D coordinate (temporal, height, width). Crucially, spatial information throughout the pipeline is determined by \textbf{positional encodings rather than sequence order}, so feeding a \textit{subset} of patches while preserving their original coordinates is an architecturally supported operation.

\subsection{PixelPrune: Patch Compression via Predictive Coding}
\label{sec:selectpatch}

Given $N$ patch--position pairs $\{(P_i, \mathbf{p}_i)\}_{i=1}^{N}$, the goal is to identify a compact subset $\mathcal{S} \subseteq \{1,\dots,N\}$ of patches to retain such that downstream task performance is preserved. As motivated in \S\ref{sec:intro}, document and GUI images exhibit extensive pixel-level redundancy, where large regions of uniform background produce many pixel-identical neighboring patches that carry no additional information.

To systematically identify such patches, we draw on \textbf{predictive coding} from lossless image compression: classical formats such as PNG~\citep{boutell1997png} and JPEG-LS~\citep{weinberger2000loco} predict each pixel from its spatial neighbors and encode the prediction residual, which is typically small and thus highly compressible. We adapt this idea from individual pixels to $Mp{\times}Mp$ patch blocks: when the residual is exactly zero (i.e., a patch block is identical to its prediction), the block is fully redundant and can be dropped entirely.

We operate on $M p \times M p$ pixel blocks ($32{\times}32$ for Qwen3-VL with ViT patch size $p{=}16$ and merge factor $M{=}2$), each corresponding to one LLM token after patch merging---the smallest independently removable unit. Blocks are scanned in a fixed order; each is predicted from its causal neighbors (those already visited) and omitted if the prediction matches. This formulation naturally provides an \textbf{exact recoverability guarantee} (\S\ref{sec:theory}): a decoder applying the same prediction rule can reconstruct all omitted patches.

Formally, for each patch $P_{r,c}$, PixelPrune omits it when the prediction is close enough:
\begin{equation}
	\text{Omit}(P_{r,c}) = \mathbb{I}\!\bigl[\,\text{dist}(P_{r,c},\, \hat{P}_{r,c}) \le \tau\,\bigr]
\end{equation}
where $\tau \ge 0$ controls matching strictness. The top-left patch $(0,0)$ has no causal neighbor and is always retained as the initial anchor.

\subsubsection{Prediction Strategy}
\label{sec:prediction_strategy}

The predictor $\hat{P}_{r,c}$ determines which spatial redundancy patterns can be captured. We consider three strategies that form a \textbf{1D-to-2D progression} within the same predictive coding framework, differing only in scan order and prediction rule:

\textbf{(i)~Raster} flattens the 2D grid into a 1D sequence in row-major order and predicts each patch from its immediate predecessor: $\hat{P}_i = P_{i-1}$. It captures horizontal continuity within rows but misses cross-row redundancy, as the last patch of row $r$ and the first of row $r{+}1$ are generally far apart spatially.

\textbf{(ii)~Serpentine} also flattens the 2D grid into a 1D sequence but alternates scan direction between rows (left-to-right on even rows, right-to-left on odd rows), using the same single-neighbor predictor. The alternating direction makes the transition between consecutive rows spatially local, partially capturing vertical continuity.

\textbf{(iii)~Pred-2D} scans in raster order but predicts each patch from three causal neighbors using a simple selection rule (analogous to the median-edge predictor in LOCO-I~\citep{weinberger2000loco}). Let $A{=}P_{r,c-1}$, $B{=}P_{r-1,c}$, and $C{=}P_{r-1,c-1}$ denote the left, upper, and upper-left neighbors (Figure~\ref{fig:loco_rule}):
\begin{equation}
	\hat{P}_{r,c} =
	\begin{cases}
		A & \text{if } C = B \neq A \quad \text{(upper and upper-left agree $\Rightarrow$ predict from left)} \\
		B & \text{if } C = A \neq B \quad \text{(left and upper-left agree $\Rightarrow$ predict from upper)} \\
		A & \text{otherwise (default: predict from left)}
	\end{cases}
\end{equation}
Intuitively, the rule picks the neighbor most likely to match the target by checking agreement among the three neighbors, simultaneously exploiting horizontal and vertical redundancy.

All three strategies operate in $O(N)$ time and guarantee exact recoverability (\S\ref{sec:theory}), but differ in compression ratio due to the redundancy patterns each can capture. We compare them in \S\ref{sec:analysis} (Table~\ref{tab:encoding_methods}) and adopt \textbf{Pred-2D} as the default, as it consistently achieves the lowest retain ratio while maintaining comparable accuracy.

\subsubsection{Matching Criterion}
\label{sec:matching_criterion}

The distance function $\text{dist}$ and threshold $\tau$ jointly define when a prediction is ``close enough'' to justify omission. Our default is $\tau{=}0$, i.e., strict pixel equality: a patch is omitted only if it is \textit{identical} to its prediction, directly mirroring lossless compression with zero information loss. This already achieves 30--80\% token reduction on document and GUI images (\S\ref{sec:doc}--\ref{sec:gui_post}). We extend to near-exact matching ($\tau{>}0$) in \S\ref{sec:theory_nearexact}, where we analyze the choice of distance metric and threshold.

\begin{figure}[t]
	\centering
	\includegraphics[width=\linewidth]{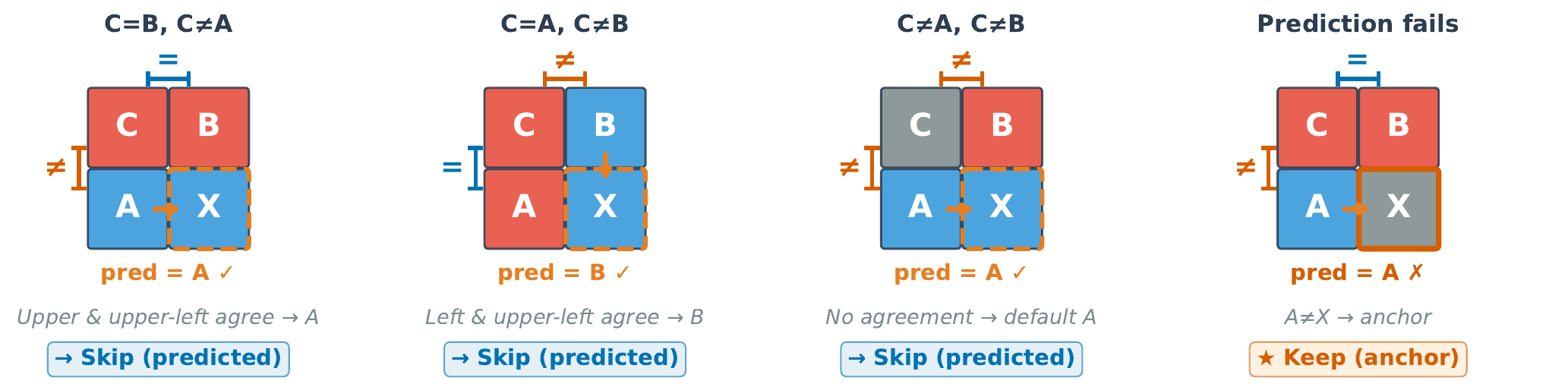}
	\caption{Illustration of PixelPrune's Pred-2D prediction. For each target patch $X$, three causal neighbors---$A$ (left), $B$ (upper), $C$ (upper-left)---determine the predicted patch $\hat{X}$. The rule selects the neighbor most likely to match~$X$: when the upper and upper-left patches agree, the target is more likely to follow the left neighbor, and vice versa. If $X$ matches its prediction, it is omitted; otherwise, it is retained.}
	\label{fig:loco_rule}
\end{figure}

\subsubsection{Position-Preserving Pipeline Integration}
\label{sec:pipeline_integration}

After predictive selection, only the $|\mathcal{S}|$ retained patch--position pairs $\{(P_k, \mathbf{p}_k)\}_{k \in \mathcal{S}}$ enter the subsequent pipeline. Each retained token keeps its \textbf{original 2D grid coordinate}---rather than being re-indexed---so that positional encodings (2D RoPE in the ViT, MRoPE in the LLM) correctly reflect spatial layout throughout the model. Because fewer patches flow through every stage---ViT, Patch Merger, and LLM---the reduction yields end-to-end acceleration across the entire pipeline.

\paragraph{Computational overhead.}
The predictive coding pass requires $O(N)$ time with constant-time operations per patch---negligible relative to the ViT's $O(N^2)$ attention. At $\tau{=}0$ (our default), pixel equality is transitive, so the entire operation can be vectorized on GPU with negligible latency. At $\tau{>}0$, maintaining reconstruction state requires sequential processing (\S\ref{sec:theory_nearexact}).

\subsection{Theoretical Guarantees}
\label{sec:theory}

\begin{proposition}[Exact Reconstruction]
\label{prop:lossless}
For $\tau = 0$, the original patch sequence $\mathcal{P} = \{P_1, \dots, P_N\}$ is exactly recoverable from the compressed representation $\mathcal{C} = \{(P_k, \mathbf{p}_k)\}_{k \in \mathcal{S}}$ via a deterministic decoder that applies the same prediction rule in the same scan order.
\end{proposition}

\textit{Proof sketch.} By induction on scan order: the anchor $(0,0)$ is always retained. For each subsequent position, omission occurs only when $P_i = \hat{P}_i$; the decoder recomputes the identical prediction from already-recovered neighbors, yielding exact reconstruction.

These guarantees hold at the pixel level; the ViT receives a shorter sequence, so its internal representations differ from the full-input case. We empirically confirm in \S\ref{sec:layer_analysis} that pixel-level pruning before the ViT achieves comparable accuracy to token reduction applied at various ViT depths.

\subsubsection{Extension to Near-Exact Matching}
\label{sec:theory_nearexact}

Setting $\tau{>}0$ extends PixelPrune beyond exact duplicates, allowing patches that differ by compression artifacts, sensor noise, or natural image gradients to be omitted as well. The prediction rule (\S\ref{sec:prediction_strategy}) remains unchanged; only the acceptance criterion is relaxed from exact equality to bounded distance. A retained patch keeps its original value; an omitted patch is replaced by its predicted value at decode time. At $\tau{=}0$ this distinction is moot, since omitted patches are identical to their predictions. At $\tau{>}0$, however, near-similarity is not transitive---a chain of pairwise-similar patches can drift arbitrarily. To prevent this, the encoder must use the \textbf{predicted value} (not the original) for each omitted patch when it serves as a neighbor for subsequent predictions, keeping encoder and decoder in identical state. This ensures the per-patch error is bounded:

\begin{proposition}[Bounded Reconstruction Error]
\label{prop:bounded}
For $\tau > 0$, when the encoder uses predicted values for omitted neighbors, every omitted patch $i \notin \mathcal{S}$ satisfies $\mathrm{dist}(P_i,\, \hat{P}_i) \le \tau$, where $\hat{P}_i$ is the prediction (which equals the decoder's output for that patch). No error accumulates across patches.
\end{proposition}

\textit{Proof sketch.} Since encoder and decoder maintain identical state (both use predicted values for omitted patches), they compute identical predictions. An omitted patch satisfies $\mathrm{dist}(P_i, \hat{P}_i) \le \tau$ by construction; since the decoder outputs exactly $\hat{P}_i$ for that patch, the error is at most $\tau$.
We ablate the choice of distance metric (MAE vs.\ Max) and threshold in \S\ref{sec:matching}.


\section{Experiments}
\label{sec:experiments}

We organize experiments along two axes: \textbf{domain} (document vs.\ GUI) and \textbf{training regime} (training-free $\to$ knowledge distillation $\to$ training from scratch). Section~\ref{sec:doc} establishes that PixelPrune works training-free on documents across three model scales. Section~\ref{sec:gui_post} extends to GUI, where training-free pruning leaves a gap on grounding tasks; knowledge distillation recovers most benchmarks but not all. Section~\ref{sec:jedi_scratch} investigates this in a controlled from-scratch setting to isolate the effect of token reduction during training. Section~\ref{sec:efficiency} quantifies the resulting efficiency gains. Section~\ref{sec:analysis} ablates key design choices across both domains.

\subsection{Experimental Setup}

\textbf{Models.} We evaluate PixelPrune on the Qwen3-VL~\citep{bai2025qwen3vltechnicalreport} model family, using the 2B, 4B, and 8B Dense variants. Document experiments use all three scales; GUI experiments and ablations use 2B unless otherwise noted.

\textbf{Benchmarks.} We evaluate on two benchmark categories. \textbf{Document understanding} (7 benchmarks): DocVQA~\citep{mathew2021docvqa}, AI2D~\citep{kembhavi2016ai2d}, ChartQA~\citep{masry2022chartqa}, InfoVQA~\citep{mathew2022infographicvqa}, OCRBench~\citep{liu2024ocrbench}, MMLongBench-Doc~\citep{ma2024mmlongbenchdoc}, and olmOCRBench~\citep{poznanski2025olmocr}. \textbf{GUI understanding}: ScreenSpot V2~\citep{wu2025screenspotv2} (Web, Mobile, Desktop) and ScreenSpot Pro~\citep{li2025screenspotpro} (Scientific, Office, OS, Creative, Development, CAD).

\textbf{Baselines.} We compare against four baselines. (1)~\textit{Full} retains all visual tokens, serving as the upper-bound reference. The remaining three operate under the same per-image token budget as PixelPrune. (2)~\textit{Random} randomly selects patches to match the budget. (3)~\textit{ConnComp} builds on ShowUI's~\citep{lin2024showui} connected-component strategy by constructing a 2D adjacency graph over pixel-identical patches and randomly retaining a subset per component to match the token budget. (4)~\textit{Resize} downsamples the input image to match the budget. Although orthogonal to patch selection and combinable with PixelPrune, we include it as an input-side compression baseline.

\textbf{Implementation Details.} All experiments use exact matching ($\tau{=}0$) at the $32{\times}32$ block granularity described in \S\ref{sec:selectpatch}. We use the default Qwen3-VL resolution settings ($\texttt{min\_pixels}{=}256{\times}256$, $\texttt{max\_pixels}{=}4096{\times}4096$). Inference experiments use 8 NVIDIA H20 GPUs with batch size 1 and greedy decoding (\texttt{max\_new\_tokens}$=$16,384 for olmOCRBench, 512 otherwise). Accuracy is evaluated using VLMEvalKit~\citep{vlmevalkit}.

\textbf{GUI Training.} All GUI training experiments share the same data source: the coordinate-containing subset filtered from Jedi~\citep{chen2025jedi}, with absolute coordinates converted to relative format for Qwen3-VL compatibility. Training uses NVIDIA H20 GPUs with DeepSpeed ZeRO-2~\citep{rajbhandari2020zero}, FlashAttention-2~\citep{dao2022flashattention}, gradient checkpointing, and data packing. Individual experiments differ in training scale (steps, GPU count) and loss function, as specified in each subsection. Detailed training configurations are provided in Appendix~\ref{app:training_details}.

\textbf{Retain Ratio.} We report \textit{dataset-level} retention ratios, $\sum_i |\mathcal{S}_i| / \sum_i N_i$, where $\mathcal{S}_i$ and $N_i$ denote the retained and total visual tokens for the $i$-th image, which better reflects actual compute savings than sample-level averages---the latter are biased toward low-resolution images that contribute few tokens.

\subsection{Training-Free Evaluation on Document Understanding}
\label{sec:doc}

Table~\ref{tab:document} presents results on 7 document understanding benchmarks across three model scales.
\textbf{First}, PixelPrune consistently matches or outperforms all baselines under the same per-image token budget. While Resize achieves reasonable accuracy via resolution scaling, PixelPrune matches or exceeds it at every scale (e.g., 73.7 vs.\ 72.8 Avg at 4B), confirming that PixelPrune preserves fine-grained details lost by resolution scaling. Random and ConnComp suffer substantial degradation (e.g., 59.8 and 64.4 vs.\ 73.7 Avg at 4B), showing that accurate redundancy identification is critical for effective token reduction.
\textbf{Second}, retention ratios vary substantially across benchmarks---MMLongBench-Doc and olmOCRBench retain 50--57\% of tokens, while InfoVQA and ChartQA retain 73--77\%---reflecting domain-specific levels of pixel redundancy.
\textbf{Third}, the accuracy gap remains small across all scales: at 8B, PixelPrune achieves 75.3 Avg versus 75.5 for Full---a 0.2-point gap with significant token reduction. These results generalize to Qwen3.5~\citep{qwen2026qwen35} with $\leq$0.8\% Avg gap across three scales (Appendix~\ref{app:qwen35}).

\begin{table*}[t]
	\caption{Training-free results on Document Understanding Benchmarks (\%). Numbers in parentheses show PixelPrune's dataset-level retain ratio; all baselines use the same per-image token budget.}
	\label{tab:document}
	\centering
	\resizebox{\textwidth}{!}{
		\begin{tabular}{ll||c|c|c|c|c|c|c||c}
			\toprule
	\textbf{Model} & \textbf{Method} & \textbf{OCRBench} & \textbf{DocVQA} & \textbf{InfoVQA} & \textbf{ChartQA} & \textbf{AI2D} & \textbf{MML-Doc} & \textbf{olmOCR} & \textbf{Avg} \\
	& & (66.1\%) & (70.1\%) & (76.7\%) & (73.1\%) & (69.2\%) & (50.3\%) & (56.9\%) & \\
	\midrule
	\midrule
	\multirow{5}{*}{Qwen3-VL-2B}
	               & Full            & \color{gray}86.9  & \color{gray}92.7   & \color{gray}72.3   & \color{gray}78.8   & \color{gray}76.6    & \color{gray}28.1       & \color{gray}48.9  & \color{gray}69.2   \\
	               & Resize          & 86.1  & \textbf{92.3}   & \textbf{72.0}   & 76.9   & \textbf{76.5}    & 26.3       & 47.8  & \textbf{68.3}   \\
	               & Random          & 76.0  & 66.8   & 62.9   & 57.3   & 72.2    & 17.0       & 34.7  & 55.3   \\
	               & ConnComp         & 77.2  & 80.5   & 66.1   & 62.8   & 73.9    & 18.4       & 37.2  & 59.4   \\
	               & \textbf{PixelPrune}      & \textbf{87.3}  & 92.0   & 71.5   & \textbf{77.0}   & 75.8    & \textbf{26.5}       & \textbf{48.0}  & \textbf{68.3}   \\
	\midrule
	\multirow{5}{*}{Qwen3-VL-4B}
	               & Full            & \color{gray}88.1  & \color{gray}94.8   & \color{gray}79.5   & \color{gray}82.9   & \color{gray}82.8    & \color{gray}36.7       & \color{gray}54.0  & \color{gray}74.1   \\
	               & Resize          & 87.6  & \textbf{94.7}   & \textbf{79.9}   & 82.2   & 81.6    & 30.9       & 53.1  & 72.8   \\
	               & Random          & 76.9  & 70.2   & 70.4   & 64.6   & 78.5    & 20.0       & 37.9  & 59.8   \\
	               & ConnComp         & 78.2  & 84.4   & 73.5   & 70.0   & 79.2    & 23.8       & 41.5  & 64.4   \\
	               & \textbf{PixelPrune}      & \textbf{88.0}  & \textbf{94.7}   & 79.2   & \textbf{82.4}   & \textbf{81.7}    & \textbf{36.3}       & \textbf{53.8}  & \textbf{73.7}   \\
	\midrule
	\multirow{5}{*}{Qwen3-VL-8B}
	               & Full            & \color{gray}90.6  & \color{gray}95.7   & \color{gray}83.1   & \color{gray}83.4   & \color{gray}82.7    & \color{gray}38.0       & \color{gray}54.7  & \color{gray}75.5   \\
	               & Resize          & 90.2  & \textbf{95.5}   & \textbf{83.2}   & 82.7   & \textbf{82.9}    & 37.1       & 53.9  & 75.1   \\
	               & Random          & 79.9  & 72.8   & 75.5   & 66.0   & 79.4    & 23.5       & 40.0  & 62.4   \\
	               & ConnComp         & 82.2  & 86.3   & 77.5   & 67.8   & 80.7    & 26.9       & 43.2  & 66.4   \\
	               & \textbf{PixelPrune}      & \textbf{90.4}  & \textbf{95.5}   & 82.8   & \textbf{82.9}   & 82.8    & \textbf{38.2}       & \textbf{54.3}  & \textbf{75.3}   \\
		\bottomrule
	\end{tabular}
}
\end{table*}

\subsection{Post-Training Adaptation on GUI Understanding}
\label{sec:gui_post}

We evaluate on the 9 GUI benchmarks with the same four baselines on Qwen3-VL-2B. We compare two settings: \textit{Training-Free} applies PixelPrune directly without fine-tuning; \textit{+KD} recovers the accuracy gap of training-free pruning on position-sensitive grounding tasks via \textbf{knowledge distillation}~\citep{hinton2015distilling}, where the same model's full-token output serves as the teacher: $\mathcal{L} = \mathcal{L}_{\text{CE}} + \mathcal{L}_{\text{KL}}(p_{\text{pred}} \| p_{\text{full}})$.

Table~\ref{tab:gui} presents results. GUI images exhibit higher redundancy than documents, with notably lower retention ratios (Web 39.7\%, Mobile 49.1\%, Desktop 61.9\%). ScreenSpot Pro reveals domain-specific variation: Scientific interfaces are most compressible (26.3\% retention), while Development, CAD, and OS retain 61--66\%. Resize suffers a significant drop on Pro (50.2$\to$43.4 Avg), particularly on Scientific (53.2$\to$37.4), where resolution scaling blurs fine-grained text critical for grounding. In contrast, PixelPrune (+KD) comes within 2.6 points of Full on Pro (47.6 vs.\ 50.2) while matching Full on V2 (90.6 vs.\ 90.6), demonstrating the advantage of adaptive selection for position-sensitive tasks. KD recovery varies across Pro sub-tasks: Office and Scientific recover well, while CAD drops from 40.2 to 30.3, likely due to sparse CAD-style data in the training set.

\begin{table*}[t]
	\caption{Post-training results on GUI Understanding Benchmark (Qwen3-VL-2B). We report grounding accuracy (\%). Numbers in parentheses show the dataset-level retain ratio. We additionally report PixelPrune (+KD) with knowledge distillation.}
	\label{tab:gui}
	\centering
	\resizebox{\textwidth}{!}{
		\begin{tabular}{l||c|c|c|c||c|c|c|c|c|c|c}
			\toprule
			& \multicolumn{4}{c||}{\textbf{ScreenSpot V2}} & \multicolumn{7}{c}{\textbf{ScreenSpot Pro}} \\
			\cmidrule(lr){2-5} \cmidrule(lr){6-12}
			\textbf{Method} & \textbf{Web} & \textbf{Mobile} & \textbf{Desktop} & \textbf{Avg} & \textbf{Dev.} & \textbf{Creative} & \textbf{CAD} & \textbf{Sci.} & \textbf{Office} & \textbf{OS} & \textbf{Avg} \\
			& (39.7\%) & (49.1\%) & (61.9\%) &  & (64.3\%) & (59.0\%) & (65.6\%) & (26.3\%) & (56.4\%) & (61.9\%) &  \\
			\midrule
			Full            & \color{gray}88.1 & \color{gray}92.0 & \color{gray}91.6 & \color{gray}90.6 & \color{gray}45.8 & \color{gray}41.6 & \color{gray}40.2 & \color{gray}53.2 & \color{gray}70.4 & \color{gray}50.0 & \color{gray}50.2 \\
			Resize          & 87.4 & 91.6 & 88.0 & 89.0 & 44.1 & 38.4 & \textbf{37.2} & 37.4 & 62.2 & \textbf{41.3} & 43.4 \\
			Random          & 48.5 & 62.7 & 68.0 & 59.7 & 29.4 & 24.6 & 26.4 & 12.6 & 37.8 & 29.6 & 26.8 \\
			ConnComp          & 69.8 & 82.0 & 75.1 & 75.7 & 35.8 & 27.9 & 28.7 & 22.1 & 47.8 & 31.1 & 32.2 \\
			PixelPrune      & 72.1 & 86.2 & 85.9 & 81.4 & 37.8 & 28.5 & 35.3 & 21.3 & 55.7 & 40.3 & 36.5 \\
			\textbf{PixelPrune (+KD)}& \textbf{87.6} & \textbf{92.4} & \textbf{91.6} & \textbf{90.6} & \textbf{50.5} & \textbf{41.9} & 30.3 & \textbf{52.8} & \textbf{68.7} & \textbf{41.3} & \textbf{47.6} \\
			\bottomrule
		\end{tabular}
	}
\end{table*}

\subsection{Training from Scratch on GUI Understanding}
\label{sec:jedi_scratch}

To isolate the effect of token reduction during training---independent of the pre-trained model's capability and training data in Section~\ref{sec:gui_post}---we train a 2B-scale model from scratch on the same Jedi data, replacing the language model of Qwen3-VL-2B-Instruct with Qwen3-1.7B-Instruct and standard cross-entropy loss (no KD). We compare three configurations: \textbf{Full$\to$Full} (standard training and inference), \textbf{PixelPrune$\to$PixelPrune} (compressed tokens for both), and \textbf{PixelPrune$\to$Full} (compressed training, full inference). Training efficiency is reported in Section~\ref{sec:training_eff}.

Table~\ref{tab:jedi_from_scratch} reveals two findings. First, \textbf{PixelPrune$\to$PixelPrune} achieves competitive accuracy, showing that token reduction during training does not significantly degrade performance (detailed convergence curves in Appendix~\ref{app:scratch_convergence}). Second, \textbf{PixelPrune$\to$Full} achieves comparable or slightly better performance than Full$\to$Full (SSv2 Avg: 85.1 vs.\ 81.8), demonstrating cross-mode generalization. We hypothesize that training with compressed tokens may encourage the model to extract information from salient patches, though the effect is small and could partly reflect variance; we leave a rigorous investigation to future work. This cross-mode compatibility offers deployment flexibility: a single checkpoint can serve low-latency (with PixelPrune) or maximum-accuracy (with full tokens) scenarios.

\begin{table*}[t]
	\caption{Training from scratch results on GUI Understanding (Qwen3-VL-2B). ``PixelPrune $\to$ Full'' = trained with PixelPrune, evaluated with full tokens.}
	\label{tab:jedi_from_scratch}
	\centering
	\resizebox{\textwidth}{!}{
		\begin{tabular}{l|c||c|c|c|c||c|c|c|c|c|c|c}
			\toprule
			& & \multicolumn{4}{c||}{\textbf{ScreenSpot V2}} & \multicolumn{7}{c}{\textbf{ScreenSpot Pro}} \\
			\cmidrule(lr){3-6} \cmidrule(lr){7-13}
			\textbf{Train} & \textbf{Infer} & \textbf{Web} & \textbf{Mobile} & \textbf{Desktop} & \textbf{Avg} & \textbf{Dev.} & \textbf{Creative} & \textbf{CAD} & \textbf{Sci.} & \textbf{Office} & \textbf{OS} & \textbf{Avg} \\
			\midrule
			Full       & Full       & 79.9 & 81.6 & 83.8 & 81.8 & \textbf{31.8} & 26.7 & 15.3 & \textbf{35.0} & 49.6 & \textbf{30.6} & 31.5 \\
			PixelPrune & PixelPrune & 79.9 & 85.6 & \textbf{88.9} & 84.8 & 30.8 & 27.3 & 15.7 & 34.6 & \textbf{52.6} & 27.6 & 31.4 \\
			PixelPrune & Full       & \textbf{80.6} & \textbf{86.0} & 88.6 & \textbf{85.1} & 30.4 & \textbf{27.9} & \textbf{16.1} & \textbf{35.0} & 50.4 & 29.6 & \textbf{31.6} \\
			\bottomrule
		\end{tabular}
	}
\end{table*}

\subsection{Efficiency Analysis}
\label{sec:efficiency}

\subsubsection{Inference Efficiency}

Table~\ref{tab:efficiency} reports prefill-stage efficiency on two benchmarks that span opposite ends of the spectrum: ScreenSpot Pro Scientific (single-page, 26.2\% retention) and MMLongBench-Doc (multi-page, 50.3\% retention with up to dozens of pages per sample). We measure on 8 NVIDIA H20 GPUs (batch size 1), discarding the first sample per GPU to exclude warmup. Across this range, PixelPrune achieves 3.0--4.2$\times$ TTFT speedup, 3.5--6.6$\times$ FLOPs reduction, and 45--63\% KV cache memory savings, while per-sample selector overhead remains negligible (see Selector column). These gains come from reduced computation across all three stages---ViT encoder, Patch Merger, and LLM---with the ViT attention saving being roughly quadratic in the pruning ratio (detailed derivations in Appendix~\ref{app:flops}).

\begin{table*}[t]
	\caption{Inference efficiency on Qwen3-VL-2B (batch size 1, H20 GPU). MML-Doc: MMLongBench-Doc; SS-Pro Sci.: ScreenSpot Pro Scientific. Selector: per-sample overhead of the pixel-level selector. Speedup factors are relative to the Full baseline.}
	\label{tab:efficiency}
	\centering
	\resizebox{\textwidth}{!}{
	\begin{tabular}{ll||cccccc|cc}
		\toprule
		& & \textbf{Selector} & \textbf{Retain} & \textbf{FLOPs} & \textbf{ViT Lat.} & \textbf{TTFT} & \textbf{KV Cache} & \textbf{FLOPs} & \textbf{TTFT} \\
		\textbf{Dataset} & \textbf{Method} & (ms) & (\%) & (TFLOPs) & (ms) & (ms) & (GB) & Speedup & Speedup \\
		\midrule
		\multirow{2}{*}{MML-Doc}
		& Full       & -- & 100.0 & 2694 & 4898 & 18477 & 13.2 & 1.0$\times$ & 1.0$\times$ \\
		& PixelPrune & 19.1 &  50.3 &  762 & 2143 &  6114 &  7.2 & \textbf{3.5$\times$} & \textbf{3.0$\times$} \\
		\midrule
		\multirow{2}{*}{SS-Pro Sci.}
		& Full       & -- & 100.0 & 70 & 625 & 833 & 0.8 & 1.0$\times$ & 1.0$\times$ \\
		& PixelPrune & 1.2 &  26.2 & 11 & 115 & 198 & 0.3 & \textbf{6.6$\times$} & \textbf{4.2$\times$} \\
		\bottomrule
	\end{tabular}
	}
\end{table*}

\subsubsection{Training Efficiency}
\label{sec:training_eff}

Because PixelPrune operates at the input level, it is fully compatible with standard training optimizations such as FlashAttention, enabling direct application during training. Table~\ref{tab:training_eff} reports the efficiency measured during the from-scratch experiment in Section~\ref{sec:jedi_scratch}. With 40.8\% average retention, PixelPrune achieves \textbf{2.0$\times$} speedup on both forward and backward passes, cutting wall-clock time from 49.1\,h to 25.3\,h (\textbf{1.9$\times$}) and peak GPU memory by \textbf{33.6\%} (57.8$\to$38.4\,GB). The modest gap between per-step and wall-clock speedup reflects fixed-cost components (data loading, gradient synchronization) that are unaffected by token count. The 19.4\,GB memory reduction can be traded for larger batch sizes or longer sequences on the same hardware.

\begin{table*}[t]
	\caption{Training efficiency measured during the from-scratch experiment in Section~\ref{sec:jedi_scratch}. Speedup is relative to Full.}
	\label{tab:training_eff}
	\centering
	\small
	\begin{tabular}{l||ccccc}
		\toprule
		\textbf{Method} & \textbf{Retain} & \textbf{Forward (s)} & \textbf{Backward (s)} & \textbf{Peak Mem (GB)} & \textbf{Wall Time (h)} \\
		\midrule
		Full       & 100.0\%  & 7.9 & 35.3 & 57.8 & 49.1 \\
		PixelPrune & 40.8\% & 3.9 & 18.0 & 38.4 & 25.3 \\
		\midrule
		\textit{Speedup} & -- & \textbf{2.0$\times$} & \textbf{2.0$\times$} & \textbf{$\downarrow$33.6\%} & \textbf{1.9$\times$} \\
		\bottomrule
	\end{tabular}
\end{table*}

\subsection{Ablation Studies}
\label{sec:analysis}

Unless otherwise noted, all ablations use Qwen3-VL-2B on six benchmarks: DocVQA, InfoVQA, and ChartQA (document, training-free) and ScreenSpot V2 Web/Mobile/Desktop (GUI, fine-tuned at a smaller scale than \S\ref{sec:gui_post}; see Appendix~\ref{app:training_details}).

\subsubsection{Comparison of Prediction Strategies}

We compare Raster (1D row-major scan), Serpentine (alternating row direction), and Pred-2D (neighbor-selective 2D prediction). All guarantee exact recoverability with $O(N)$ complexity. Table~\ref{tab:encoding_methods} reports retain ratio and accuracy.

The strategies form a 1D-to-2D progression. Serpentine improves over Raster by making row transitions spatially local, yielding modest retention reductions (e.g., ChartQA $76.7{\to}74.9$\%). Pred-2D further lowers retention by additionally exploiting vertical redundancy, with the largest gains on GUI benchmarks where rectangular panels span multiple rows (e.g., GUI-Web $42.5{\to}39.7$\%, GUI-Desktop $64.4{\to}61.9$\%). Despite the progressive compression gains, all three strategies achieve comparable accuracy, confirming that the predictor design primarily affects compression efficiency rather than downstream task performance. We therefore adopt \textbf{Pred-2D} as the default for the best compression--accuracy trade-off.

\begin{table*}[t]
	\caption{Comparison of prediction strategies.}
	\label{tab:encoding_methods}
	\centering
	\resizebox{\textwidth}{!}{
	\begin{tabular}{l|cc|cc|cc|cc|cc|cc}
		\toprule
		& \multicolumn{6}{c|}{\textbf{Document}} & \multicolumn{6}{c}{\textbf{GUI}} \\
		\cmidrule(lr){2-7} \cmidrule(lr){8-13}
		& \multicolumn{2}{c|}{\textbf{DocVQA}} & \multicolumn{2}{c|}{\textbf{InfoVQA}} & \multicolumn{2}{c|}{\textbf{ChartQA}} & \multicolumn{2}{c|}{\textbf{Web}} & \multicolumn{2}{c|}{\textbf{Mobile}} & \multicolumn{2}{c}{\textbf{Desktop}} \\
		\cmidrule(lr){2-3} \cmidrule(lr){4-5} \cmidrule(lr){6-7} \cmidrule(lr){8-9} \cmidrule(lr){10-11} \cmidrule(lr){12-13}
		\textbf{Method} & \textbf{Ret.} & \textbf{Acc.} & \textbf{Ret.} & \textbf{Acc.} & \textbf{Ret.} & \textbf{Acc.} & \textbf{Ret.} & \textbf{Acc.} & \textbf{Ret.} & \textbf{Acc.} & \textbf{Ret.} & \textbf{Acc.} \\
		\midrule
		Raster              & 70.1 & 92.0 & 77.8 & 71.4 & 76.7 & \textbf{77.6} & 42.5 & 87.2 & 50.5 & 91.8 & 64.4 & \textbf{90.1} \\
		Serpentine          & \textbf{70.0} & \textbf{92.1} & 77.4 & 71.2 & 74.9 & 77.1 & 41.8 & \textbf{88.6} & 50.1 & 91.2 & 64.1 & 88.9 \\
		\textbf{Pred-2D}  & 70.1 & 92.0 & \textbf{76.7} & \textbf{71.5} & \textbf{73.1} & 77.0 & \textbf{39.7} & 86.5 & \textbf{49.1} & \textbf{92.2} & \textbf{61.9} & \textbf{90.1} \\
		\bottomrule
	\end{tabular}
	}
\end{table*}

\subsubsection{Effect of Matching Threshold}
\label{sec:matching}

We compare two distance metrics---MAE (mean absolute error) and Max (maximum per-pixel difference), both normalized to $[0,1]$---across several thresholds $\tau$ (Table~\ref{tab:nearexact}).

At $\tau{=}0$ (exact matching), accuracy is already competitive with all $\tau{>}0$ configurations on most benchmarks. When $\tau{>}0$, Max consistently outperforms MAE, which degrades rapidly as it averages per-pixel differences and masks critical outliers (e.g., ChartQA 76.6\% vs 43.5\% at $\tau{=}0.05$). Max maintains strong accuracy even at high compression: at $\tau{=}0.10$, DocVQA retention drops to 45.1\% with less than 1\% accuracy loss. ChartQA is the most sensitive to $\tau$, as chart understanding requires spatial layout reasoning beyond pure text recognition. This suggests that fine-tuning with PixelPrune may be needed for such tasks under aggressive pruning. We further evaluate on general-domain images in Appendix~\ref{app:general}, covering both exact and near-exact matching.

\subsubsection{Token Reduction at Different ViT Depths}
\label{sec:layer_analysis}

We ablate the pruning position within the Qwen3-VL-2B ViT encoder (24 blocks; DeepStack at blocks 5, 11, 17). All settings apply the same retain mask at different depths---after the 0th, 4th, 17th, or final ViT block, or before the ViT (PixelPrune default)---so the retained tokens share the same spatial positions but carry features from different depths.

\begin{wraptable}{r}{0.48\textwidth}
	\vspace{-2em}
	\caption{Token reduction at different ViT depths.}
	\label{tab:layer_pruning}
	\centering
	\footnotesize
	\setlength{\tabcolsep}{3pt}
	\begin{tabular}{l|ccc|ccc}
		\toprule
		& \multicolumn{3}{c|}{\textbf{Doc.}} & \multicolumn{3}{c}{\textbf{GUI}} \\
		\cmidrule(lr){2-4} \cmidrule(lr){5-7}
		\textbf{Depth} & \textbf{Doc} & \textbf{Info} & \textbf{Chart} & \textbf{Web} & \textbf{Mob.} & \textbf{Desk.} \\
		\midrule
		After ViT   & \textbf{92.2} & 71.0 & 76.8 & 88.1 & 92.2 & 90.4 \\
		After 17th  & 91.9 & 71.1 & 76.1 & 86.7 & 92.4 & \textbf{92.2} \\
		After 4th   & \textbf{92.2} & 71.4 & \textbf{77.0} & \textbf{89.9} & 92.6 & 89.8 \\
		After 0th   & 92.1 & 71.3 & 76.9 & 86.7 & \textbf{93.0} & 90.4 \\
		\midrule
		\textbf{Before ViT} & 92.0 & \textbf{71.5} & \textbf{77.0} & 86.5 & 92.2 & 90.1 \\
		\bottomrule
	\end{tabular}
	\vspace{-2em}
\end{wraptable}

Table~\ref{tab:layer_pruning} reveals two findings.
\textbf{First}, accuracy is largely insensitive to pruning depth: no single layer consistently dominates, and PixelPrune (before ViT) performs on par with all post-ViT variants despite the latter preserving richer learned features for the kept tokens. For structured images, \emph{which} patches to keep matters more than \emph{where} in the network they are pruned.
\textbf{Second}, earlier pruning yields greater computational savings. After-ViT/17th/4th/0th pruning saves 0\%/29\%/83\%/96\% of ViT FLOPs for removed tokens, respectively; only PixelPrune (before ViT) saves 100\%, delivering full-pipeline acceleration.

\begin{table*}[t]
	\caption{Distance metrics and thresholds. Document = training-free; GUI = +KD model from Section~\ref{sec:gui_post}. $\tau$ is normalized to $[0,1]$.}
	\label{tab:nearexact}
	\centering
	\resizebox{\textwidth}{!}{
	\begin{tabular}{ll|cc|cc|cc|cc|cc|cc}
		\toprule
		& & \multicolumn{6}{c|}{\textbf{Document}} & \multicolumn{6}{c}{\textbf{GUI}} \\
		\cmidrule(lr){3-8} \cmidrule(lr){9-14}
		& & \multicolumn{2}{c|}{\textbf{DocVQA}} & \multicolumn{2}{c|}{\textbf{InfoVQA}} & \multicolumn{2}{c|}{\textbf{ChartQA}} & \multicolumn{2}{c|}{\textbf{Web}} & \multicolumn{2}{c|}{\textbf{Mobile}} & \multicolumn{2}{c}{\textbf{Desktop}} \\
		\cmidrule(lr){3-4} \cmidrule(lr){5-6} \cmidrule(lr){7-8} \cmidrule(lr){9-10} \cmidrule(lr){11-12} \cmidrule(lr){13-14}
		\textbf{Distance} & \textbf{$\tau$} & \textbf{Ret.} & \textbf{Acc.} & \textbf{Ret.} & \textbf{Acc.} & \textbf{Ret.} & \textbf{Acc.} & \textbf{Ret.} & \textbf{Acc.} & \textbf{Ret.} & \textbf{Acc.} & \textbf{Ret.} & \textbf{Acc.} \\
		\midrule
		\multicolumn{2}{l|}{\color{gray}Full (100\%)} & \color{gray}100 & \color{gray}92.7 & \color{gray}100 & \color{gray}72.3 & \color{gray}100 & \color{gray}78.8 & \color{gray}100 & \color{gray}88.1 & \color{gray}100 & \color{gray}92.0 & \color{gray}100 & \color{gray}91.6 \\
		\multicolumn{2}{l|}{Exact ($\tau$=0)} & 70.1 & 92.0 & 76.7 & 71.5 & 73.1 & 77.0 & 39.7 & 87.6 & 49.1 & 92.4 & 61.9 & 91.6 \\
		\midrule
		\multirow{4}{*}{MAE}
		& 0.005 & 61.8 & 92.0 & 66.7 & 70.4 & 53.1 & 74.4 & 36.0 & 87.4 & 35.9 & 92.2 & 42.7 & 91.6 \\
		& 0.020 & 37.0 & 90.9 & 56.9 & 68.4 & 35.0 & 69.4 & 31.3 & 86.7 & 27.1 & 91.2 & 30.7 & 88.9 \\
		& 0.050 & 28.6 & 86.8 & 47.1 & 65.8 & 23.3 & 43.5 & 25.7 & 81.2 & 20.6 & 88.2 & 23.3 & 82.0 \\
		& 0.100 & 21.2 & 69.5 & 35.1 & 55.4 & 13.2 & 25.8 & 17.8 & 63.2 & 14.4 & 79.6 & 14.9 & 59.6 \\
		\midrule
		\multirow{4}{*}{Max}
		& 0.005 & 69.8 & \textbf{92.1} & 75.5 & 71.2 & 72.6 & \textbf{77.2} & 39.2 & 87.9 & 44.7 & 92.4 & 53.9 & 91.0 \\
		& 0.020 & 66.1 & 91.8 & 71.7 & \textbf{71.6} & 69.3 & 76.4 & 38.0 & 87.9 & 38.2 & \textbf{92.6} & 45.7 & 91.6 \\
		& 0.050 & 53.1 & 91.8 & 68.0 & 71.1 & 62.2 & 76.6 & 37.1 & \textbf{88.3} & 34.4 & 92.2 & 40.1 & 91.9 \\
		& 0.100 & 45.1 & 91.7 & 62.4 & 70.5 & 56.8 & 74.8 & 35.8 & 87.9 & 30.1 & 92.0 & 35.9 & \textbf{93.1} \\
		\bottomrule
	\end{tabular}
	}
\end{table*}

\section{Conclusion}
\label{sec:conclusion}
We presented PixelPrune, a pixel-level adaptive patch pruning method that removes redundant visual tokens \emph{before} the ViT encoder via 2D predictive coding. On document benchmarks, PixelPrune matches full-token accuracy with 23--50\% token reduction across three model scales without any training; on GUI benchmarks, light post-training and training-from-scratch settings recover the accuracy gap while delivering significant training acceleration. Across document and GUI benchmarks, PixelPrune achieves up to 4.2$\times$ inference speedup, 1.9$\times$ training acceleration, and 33.6\% memory reduction, while adapting gracefully to general-domain images (Appendix~\ref{app:general}).

\bibliographystyle{unsrtnat}
\bibliography{references}

\appendix
\pdfbookmark[0]{Appendices}{appendices}   
\bookmarksetup{level=1}                   
\section{Pixel-Level Redundancy and Resolution Statistics}
\label{app:redundancy}

Table~\ref{tab:redundancy_stats} summarizes the pixel-level patch redundancy and resolution characteristics of all benchmarks. We divide each image into non-overlapping $32{\times}32$ blocks and count the number of \textbf{pixel-unique} blocks (those whose pixel content appears exactly once in the image) versus \textbf{duplicate} blocks (those sharing identical pixel content with at least one other block in the same image). \textbf{Duplicate Ratio} is the dataset-level fraction of non-unique blocks, $1 - \sum_i U_i / \sum_i N_i$, where $U_i$ and $N_i$ are the unique and total block counts for the $i$-th image. This global dedup statistic characterizes the overall pixel-level redundancy of each benchmark; PixelPrune's actual retain ratio (\S\ref{sec:experiments}) is higher because its causal prediction can only exploit redundancy along the scan order.

We also report the \textbf{Mean Resolution} (width$\times$height in pixels) of each benchmark after Qwen3-VL's default resolution preprocessing (\texttt{min\_pixels}=256$\times$256, \texttt{max\_pixels}=4096$\times$4096). Resolution directly affects the number of visual tokens: higher-resolution benchmarks produce more patches and thus more opportunities for redundancy, but the relationship is not strictly monotonic---content density matters as well.

On average across benchmarks, GUI images exhibit higher duplication (51.4\% mean across 9 benchmarks) than documents (39.7\% mean across 7 benchmarks), driven by large solid-color regions such as backgrounds, panels, and status bars. The range spans from SSPro-Scientific (77.6\%, paper-like backgrounds) to SSPro-CAD (40.3\%, dense content). Among documents, MMLongBench-Doc leads (52.4\%, multi-page white margins) while InfoVQA is lowest (28.6\%, dense infographics).

\begin{table*}[ht!]
	\caption{Pixel-level patch redundancy and resolution statistics (patch size $32{\times}32$, exact match). Duplicate Ratio = fraction of non-unique patches (global dedup within each image).}
	\label{tab:redundancy_stats}
	\centering
	\resizebox{\textwidth}{!}{
		\begin{tabular}{l||r|c|r}
			\toprule
			\textbf{Benchmark} & \textbf{\#Images} & \textbf{Mean Resolution} & \textbf{Duplicate Ratio (\%)} \\
			\midrule
			\multicolumn{4}{c}{\textit{Document Understanding}} \\
			\midrule
			OCRBench        & 1{,}000 &  615${\times}$732  & 38.5 \\
			DocVQA          & 1{,}286 & 1775${\times}$2082 & 33.4 \\
			InfoVQA         &    500  & 1141${\times}$2929 & 28.6 \\
			ChartQA         & 2{,}500 &  768${\times}$584  & 40.8 \\
			AI2D            & 1{,}204 &  609${\times}$486  & 37.7 \\
			MMLongBench-Doc & 5{,}784 & 1284${\times}$1405 & 52.4 \\
			olmOCRBench     & 1{,}403 & 1556${\times}$2007 & 46.6 \\
			\midrule
			\multicolumn{4}{c}{\textit{GUI Understanding}} \\
			\midrule
			SSv2-Web         & 251 & 2397${\times}$1348 & 63.6 \\
			SSv2-Mobile      & 248 & 1387${\times}$2339 & 56.2 \\
			SSv2-Desktop     & 231 & 1369${\times}$907  & 42.3 \\
			SSPro-Dev.       & 299 & 3624${\times}$2119 & 43.0 \\
			SSPro-Creative   & 341 & 3040${\times}$1616 & 46.6 \\
			SSPro-CAD        & 261 & 3448${\times}$1190 & 40.3 \\
			SSPro-Scientific & 254 & 2716${\times}$1616 & 77.6 \\
			SSPro-Office     & 230 & 2980${\times}$1778 & 51.0 \\
			SSPro-OS         & 196 & 3926${\times}$2118 & 41.7 \\
			\bottomrule
		\end{tabular}
	}
\end{table*}

\section{Training Details}
\label{app:training_details}

Post-training KD (\S\ref{sec:gui_post}) uses 32 GPUs $\times$ 36 samples/GPU for 2{,}000 steps. Training from scratch (\S\ref{sec:jedi_scratch}) halves the per-GPU batch size to 18 to avoid OOM under the Full configuration (which retains all visual tokens), and accordingly doubles the steps to 4{,}000 to match the total sample budget. Ablations on prediction strategy (\S\ref{sec:analysis}) and ViT depth (\S\ref{sec:layer_analysis}) each use 8 GPUs $\times$ 36 samples/GPU for 1{,}000 steps.

\section{Cross-Architecture Validation: Qwen3.5}
\label{app:qwen35}

To verify generalization beyond Qwen3-VL, we evaluate PixelPrune on Qwen3.5 (training-free, document benchmarks). Qwen3.5 shares the NaViT-style ViT with separate spatial positional encodings, making it directly compatible with PixelPrune. Notably, Qwen3.5 employs a \textbf{hybrid attention} mechanism in the LLM: it interleaves standard full attention layers with linear attention layers. The linear attention layers lack explicit positional encoding injection and rely only on naturally decaying temporal ordering for sequence awareness. This architectural difference tests whether PixelPrune's position-preserving design remains effective when the LLM's positional signal is weaker and partially implicit.

Table~\ref{tab:qwen35_doc} shows that PixelPrune maintains small accuracy gaps across all scales ($\leq$0.8\% Avg), confirming generalization across NaViT-based architectures.

\begin{table*}[t]
	\caption{Training-free results on Document Understanding Benchmarks using Qwen3.5. Format follows Table~\ref{tab:document}. Numbers in parentheses show the dataset-level retain ratio of PixelPrune.}
	\label{tab:qwen35_doc}
	\centering
	\resizebox{\textwidth}{!}{
		\begin{tabular}{ll||c|c|c|c|c|c|c||c}
			\toprule
	\textbf{Model} & \textbf{Method} & \textbf{OCRBench} & \textbf{DocVQA} & \textbf{InfoVQA} & \textbf{ChartQA} & \textbf{AI2D} & \textbf{MML-Doc} & \textbf{olmOCR} & \textbf{Avg} \\
	& & (66.1\%) & (70.1\%) & (76.7\%) & (73.1\%) & (69.2\%) & (50.3\%) & (56.9\%) & \\
	\midrule
	\midrule
	\multirow{2}{*}{Qwen3.5-2B}
	               & Full            & \color{gray}87.3 & \color{gray}93.5 & \color{gray}75.4 & \color{gray}78.8 & \color{gray}79.6 & \color{gray}32.4 & \color{gray}49.9 & \color{gray}71.0 \\
	               & \textbf{PixelPrune}      & 87.9 & 92.5 & 74.3 & 77.9 & 78.7 & 31.6 & 48.7 & 70.2 \\
	\midrule
	\multirow{2}{*}{Qwen3.5-4B}
	               & Full            & \color{gray}86.6 & \color{gray}95.4 & \color{gray}81.6 & \color{gray}85.1 & \color{gray}84.9 & \color{gray}45.4 & \color{gray}53.5 & \color{gray}76.1 \\
	               & \textbf{PixelPrune}      & 86.6 & 95.3 & 80.7 & 84.7 & 84.7 & 45.4 & 53.6 & 75.9 \\
	\midrule
	\multirow{2}{*}{Qwen3.5-9B}
	               & Full            & \color{gray}88.3 & \color{gray}95.7 & \color{gray}84.4 & \color{gray}86.2 & \color{gray}86.7 & \color{gray}53.1 & \color{gray}54.1 & \color{gray}78.4 \\
	               & \textbf{PixelPrune}      & 88.0 & 95.5 & 83.7 & 85.2 & 86.5 & 51.9 & 53.3 & 77.7 \\
		\bottomrule
	\end{tabular}
}
\end{table*}

\section{Training from Scratch: Detailed Convergence}
\label{app:scratch_convergence}

Table~\ref{tab:scratch_convergence} reports intermediate evaluation results during training from scratch (Section~\ref{sec:jedi_scratch}), comparing \textbf{Full} and \textbf{PixelPrune} configurations. Each configuration uses the same visual token setting for both training and inference: Full trains and evaluates with all tokens, while PixelPrune trains and evaluates with compressed tokens.

Both configurations converge comparably, reaching similar accuracy by step 2{,}400 and plateauing around step 2{,}400--3{,}200. PixelPrune achieves slightly higher final performance (mean 58.1 vs.\ 56.6 at step 4{,}000), possibly because reduced token count provides implicit regularization.

\begin{table*}[t]
	\caption{Training from scratch convergence: intermediate evaluation results every 400 steps. We report grounding accuracy (\%) on ScreenSpot V2 (SSv2) and ScreenSpot Pro (SSPro). ``Full'' = trained and evaluated with all tokens; ``PixelPrune'' = trained and evaluated with compressed tokens.}
	\label{tab:scratch_convergence}
	\centering
	\resizebox{\textwidth}{!}{
		\begin{tabular}{l|r||c|c|c|c||c|c|c|c|c|c|c}
			\toprule
			& & \multicolumn{4}{c||}{\textbf{ScreenSpot V2}} & \multicolumn{7}{c}{\textbf{ScreenSpot Pro}} \\
			\cmidrule(lr){3-6} \cmidrule(lr){7-13}
			\textbf{Mode} & \textbf{Step} & \textbf{Web} & \textbf{Mobile} & \textbf{Desktop} & \textbf{Avg} & \textbf{Dev.} & \textbf{Cre.} & \textbf{CAD} & \textbf{Sci.} & \textbf{Office} & \textbf{OS} & \textbf{Avg} \\
			\midrule
			\multirow{10}{*}{Full}
			& 400  & 66.1 & 78.2 & 75.1 & 73.2 & 14.7 & 14.1 &  5.0 & 20.5 & 22.6 & 10.7 & 14.6 \\
			& 800  & 74.4 & 84.4 & 82.3 & 80.4 & 22.4 & 19.6 & 12.6 & 24.4 & 29.1 & 15.3 & 20.6 \\
			& 1200 & 78.0 & 87.2 & 85.9 & 83.7 & 24.7 & 22.6 & 14.9 & 29.5 & 40.9 & 24.5 & 26.2 \\
			& 1600 & 78.7 & 88.4 & 85.0 & 84.1 & 29.8 & 28.2 & 13.4 & 31.5 & 42.6 & 25.5 & 28.5 \\
			& 2000 & 78.7 & 84.4 & 83.5 & 82.2 & 25.1 & 23.5 & 14.2 & 27.6 & 40.9 & 25.0 & 26.0 \\
			& 2400 & 79.4 & 84.4 & 87.4 & 83.8 & 29.4 & 24.9 & 16.9 & 34.3 & 48.3 & 29.1 & 30.5 \\
			& 2800 & 80.1 & 78.6 & 85.6 & 81.5 & 32.1 & 26.7 & 15.3 & 35.4 & 49.1 & 30.6 & 31.5 \\
			& 3200 & 79.9 & 84.8 & 85.0 & 83.2 & 31.8 & 26.1 & 16.5 & 33.9 & 49.1 & 27.6 & 30.8 \\
			& 3600 & 80.1 & 83.4 & 85.0 & 82.9 & 31.1 & 27.3 & 16.1 & 34.3 & 48.7 & 28.6 & 31.0 \\
			& 4000 & 79.9 & 81.6 & 83.8 & 81.8 & 31.8 & 26.7 & 15.3 & 35.0 & 49.6 & 30.6 & 31.5 \\
			\midrule
			\multirow{10}{*}{PixelPrune}
			& 400  & 61.1 & 81.4 & 76.3 & 73.0 & 13.4 & 10.9 &  4.6 & 15.7 & 17.8 & 11.2 & 12.3 \\
			& 800  & 72.3 & 86.0 & 82.3 & 80.2 & 21.7 & 19.1 &  7.7 & 26.4 & 29.6 & 13.8 & 19.7 \\
			& 1200 & 75.5 & 87.6 & 84.1 & 82.4 & 22.4 & 20.2 & 13.0 & 29.5 & 36.1 & 21.9 & 23.9 \\
			& 1600 & 76.2 & 87.0 & 87.1 & 83.5 & 27.4 & 26.4 & 14.9 & 33.9 & 42.6 & 26.0 & 28.5 \\
			& 2000 & 78.3 & 87.2 & 86.8 & 84.1 & 27.4 & 23.5 & 14.6 & 30.7 & 41.7 & 28.1 & 27.7 \\
			& 2400 & 79.6 & 85.2 & 88.3 & 84.4 & 30.1 & 27.6 & 15.7 & 34.6 & 50.9 & 27.0 & 31.0 \\
			& 2800 & 79.4 & 86.4 & 86.8 & 84.2 & 27.8 & 27.9 & 14.2 & 34.3 & 50.0 & 29.1 & 30.5 \\
			& 3200 & 80.8 & 87.4 & 87.7 & 85.3 & 31.4 & 27.3 & 14.6 & 34.3 & 51.3 & 28.6 & 31.2 \\
			& 3600 & 79.6 & 85.8 & 88.0 & 84.5 & 30.1 & 28.4 & 16.1 & 34.3 & 50.4 & 30.1 & 31.6 \\
			& 4000 & 79.9 & 85.6 & 88.9 & 84.8 & 30.8 & 27.3 & 15.7 & 34.6 & 52.6 & 27.6 & 31.4 \\
			\bottomrule
		\end{tabular}
	}
\end{table*}

\section{Detailed FLOPs Analysis}
\label{app:flops}

We derive the FLOPs formula for the Qwen3-VL architecture, which comprises three stages: ViT encoding, Patch Merger (including DeepStack), and LLM processing. For an input image with $N$ patches ($N = (H/p) \times (W/p)$ for ViT patch size $p$) and $T$ text tokens:

\begin{align}
\text{FLOPs}_{\text{total}} &=
\underbrace{
    L_v \!\bigl[\, N\!\left(8D_v^2 + 4D_v D_{f,v}\right) + 4N^2 D_v \,\bigr]
}_{\text{(a) ViT Encoder}} \notag\\
&\quad+\;\underbrace{
    L_m \,\frac{N}{M^2}\!\left(2D_{\text{in}}^2 + 2 D_{\text{in}} D_l\right)
}_{\text{(b) Patch Merger \& DeepStack}}
\;+\;\underbrace{
    L_l \!\bigl[\, N_l C_l + 4 N_l^2 D_l \,\bigr]
}_{\text{(c) LLM}}
\label{eq:flops_total}
\end{align}

where $N_l = N/M^2 + T$ is the total LLM sequence length ($M{=}2$ is the spatial merge factor, $T$ is the number of text tokens), $D_{f,v}$ is the ViT FFN dimension, $D_{\text{in}} = M^2 \cdot D_v$ is the Patch Merger input dimension, $L_m = 1 + |\mathcal{I}_{\text{ds}}|$ accounts for the main merger plus DeepStack mergers, and $C_l$ is the per-token linear cost of the LLM (defined in Eq.~\ref{eq:llm_per_token} below).

\paragraph{(a) ViT Encoder.} Each of the $L_v$ ViT blocks contains multi-head self-attention (MHSA) and a feed-forward network (FFN). Linear projections contribute $8D_v^2 + 4 D_v D_{f,v}$ FLOPs per token (QKV, output, and two FFN layers); attention maps contribute $4N^2 D_v$ FLOPs per layer. The ViT computes attention \emph{per image} (controlled by \texttt{cu\_seqlens}), so $N$ refers to the patch count of a single image.

\paragraph{(b) Patch Merger \& DeepStack.} After $M{\times}M$ spatial merge, each token has dimension $D_{\text{in}} = M^2 D_v$. The merger MLP consists of \texttt{fc1}: $D_{\text{in}} \to D_{\text{in}}$ ($2 \tfrac{N}{M^2} D_{\text{in}}^2$ FLOPs) and \texttt{fc2}: $D_{\text{in}} \to D_l$ ($2 \tfrac{N}{M^2} D_{\text{in}} D_l$ FLOPs). The main merger runs once on the final ViT output; each DeepStack merger runs on an intermediate layer's output, giving a total count of $L_m$.

\paragraph{(c) LLM.} Qwen3-VL uses Grouped-Query Attention (GQA) with $n_q$ query heads and $n_{kv}$ key-value heads, and a gated MLP (SwiGLU). The per-token linear cost per decoder layer is:
\begin{equation}
C_l = \underbrace{2 D_l (2n_q + 2 n_{kv}) d_h}_{\mathclap{\text{QKV + output proj.}}} \;+\; \underbrace{6 D_l D_{f,l}}_{\mathclap{\text{Gated MLP (gate + up + down)}}}
\label{eq:llm_per_token}
\end{equation}
where $d_h = D_l / n_q$ and $D_{f,l}$ is the LLM FFN dimension. The Q and output projections each have dimension $D_l \times D_l$, while K and V projections have dimension $D_l \times (n_{kv} d_h)$. Attention maps add $4 N_l^2 D_l$ FLOPs per layer (since $n_q d_h = D_l$, this is independent of $n_{kv}$).

\paragraph{Architecture parameters.} For Qwen3-VL-2B: (a)~ViT: $L_v{=}24$, $D_v{=}1024$, $D_{f,v}{=}4096$, $n_h{=}16$; (b)~Patch Merger: $M{=}2$, $D_{\text{in}} = 4096$, $D_l{=}2048$, $|\mathcal{I}_{\text{ds}}|{=}3$; (c)~LLM: $L_l{=}28$, $D_l{=}2048$, $D_{f,l}{=}6144$, $n_q{=}16$, $n_{kv}{=}8$, $d_h{=}128$.

\paragraph{FLOPs savings with PixelPrune.} Let $N_s = |\mathcal{S}| \le N$ denote the number of retained patches after PixelPrune (\S\ref{sec:selectpatch}). Replacing $N$ with $N_s$ in the formula above yields: (i)~\textbf{ViT}: attention FLOPs reduce from $O(N^2)$ to $O(N_s^2)$ (quadratic saving), linear FLOPs reduce proportionally; (ii)~\textbf{Patch Merger}: scales linearly, from $O(N/M^2)$ to $O(N_s/M^2)$; (iii)~\textbf{LLM}: the sequence length reduces from $N_l$ to $N_l' = N_s/M^2+T$, so attention reduces from $O(N_l^2)$ to $O(N_l'^2)$. When $N_s \ll N$, the end-to-end reduction is substantial across all three stages. Plugging the architecture parameters and measured $N_s$ into Eq.~\ref{eq:flops_total} yields the theoretical FLOPs reported in Table~\ref{tab:efficiency}.

\section{General-Domain Evaluation}
\label{app:general}

Although PixelPrune primarily targets document and GUI scenarios, we evaluate whether it degrades performance on general-domain images with rich textures and fewer uniform regions (Table~\ref{tab:general}; six benchmarks, Qwen3-VL-2B, training-free).

Exact matching ($\tau{=}0$) retains 80--100\% of tokens with $\leq$1.1\% accuracy change. Near-exact matching further reduces retention: at $\tau{=}0.05$, retention drops to 72--89\% yet accuracy stays within 0.6\% on most benchmarks. Even at $\tau{=}0.10$ (68--83\% retention), degradation remains modest ($\leq$2\%).

\begin{table*}[t]
	\caption{General-domain evaluation on Qwen3-VL-2B (training-free, Max distance metric). Each benchmark reports retain ratio (\%) and accuracy (\%). Bold indicates the best accuracy among PixelPrune variants per benchmark.}
	\label{tab:general}
	\centering
	\resizebox{\textwidth}{!}{
	\begin{tabular}{l||cc|cc|cc|cc|cc|cc}
		\toprule
		& \multicolumn{2}{c|}{\textbf{MMBench}} & \multicolumn{2}{c|}{\textbf{MMMU}} & \multicolumn{2}{c|}{\textbf{MMStar}} & \multicolumn{2}{c|}{\textbf{RealWorldQA}} & \multicolumn{2}{c|}{\textbf{HR-Bench}} & \multicolumn{2}{c}{\textbf{V*Bench}} \\
		\cmidrule(lr){2-3} \cmidrule(lr){4-5} \cmidrule(lr){6-7} \cmidrule(lr){8-9} \cmidrule(lr){10-11} \cmidrule(lr){12-13}
		\textbf{$\tau$} & \textbf{Ret.} & \textbf{Acc.} & \textbf{Ret.} & \textbf{Acc.} & \textbf{Ret.} & \textbf{Acc.} & \textbf{Ret.} & \textbf{Acc.} & \textbf{Ret.} & \textbf{Acc.} & \textbf{Ret.} & \textbf{Acc.} \\
		\midrule
		Full (no dedup) & \color{gray}100 & \color{gray}77.1 & \color{gray}100 & \color{gray}50.3 & \color{gray}100 & \color{gray}55.1 & \color{gray}100 & \color{gray}66.3 & \color{gray}100 & \color{gray}70.6 & \color{gray}100 & \color{gray}71.2 \\
		\midrule
		0 (exact) & 94.7 & \textbf{77.1} & 80.1 & 49.2 & 81.2 & 55.1 & 97.7 & 66.3 & 96.3 & 71.0 & 99.7 & 71.2 \\
		0.005     & 94.2 & 76.9 & 79.2 & 49.7 & 80.5 & 55.1 & 97.6 & 66.1 & 94.9 & 70.8 & 99.3 & 71.2 \\
		0.02      & 91.8 & 76.9 & 76.1 & 49.1 & 77.2 & \textbf{55.3} & 96.6 & 65.8 & 84.6 & 71.4 & 94.0 & \textbf{73.8} \\
		0.05      & 88.0 & 76.5 & 72.4 & \textbf{50.3} & 73.3 & 55.0 & 89.2 & 66.1 & 77.3 & \textbf{72.0} & 87.7 & 72.8 \\
		0.10      & 83.5 & 76.5 & 68.0 & 48.2 & 68.8 & 54.8 & 75.4 & \textbf{67.1} & 70.1 & 71.9 & 82.3 & 71.2 \\
		\bottomrule
	\end{tabular}
	}
\end{table*}

Overall, PixelPrune's content-adaptive compression gracefully handles general-domain images: exact matching preserves nearly all tokens with negligible accuracy change, while near-exact matching provides moderate compression when needed.

\end{document}